\documentclass{article}

\usepackage[preprint]{nips_2018}

\usepackage[utf8]{inputenc} % allow utf-8 input
\usepackage[T1]{fontenc}    % use 8-bit T1 fonts
\usepackage{times}
\usepackage{graphicx}
\usepackage{algorithm2e} % algorithms
\usepackage{color}
\usepackage{multirow}
\usepackage{hyperref}       % hyperlinks
\usepackage{url}            % simple URL typesetting
\usepackage{booktabs}       % professional-quality tables
\usepackage{amsfonts}       % blackboard math symbols
\usepackage{amsmath}
\usepackage{amssymb}
\usepackage{amsthm}
\usepackage{nicefrac}       % compact symbols for 1/2, etc.
\usepackage{microtype}      % micro typography
\usepackage{booktabs}
\usepackage{caption}
\usepackage[flushleft]{threeparttable}

\newtheorem{thm}{Theorem}
\newtheorem{crl}{Corollary}

\newcommand{\xbold}{\boldsymbol{x}}
\newcommand{\zbold}{\boldsymbol{z}}
\newcommand{\abold}{\boldsymbol{a}}
\newcommand{\bbold}{\boldsymbol{b}}
\newcommand{\ubold}{\boldsymbol{u}}
\newcommand{\onebold}{\boldsymbol{1}}
\newcommand{\y}{\boldsymbol{y}}
\newcommand{\w}{\boldsymbol{w}}
\newcommand{\Ftr}{\mathbf{\Phi}}
\newcommand{\0}{\boldsymbol{0}}
\newcommand{\Y}{\mathcal{Y}}
\newcommand{\D}{\mathcal{D}}
\newcommand{\e}{\ell}
\newcommand{\1}[1]{\boldsymbol{1}_{#1}}
\newcommand{\indicator}[1]{\boldsymbol{1}{[#1]}}
\newcommand{\sign}{\textrm{sign}}
\title{Cumulative Sum Ranking}

\author{
    Ruy Luiz Milidiú\\
    Departmento de Informática\\
      Pontifícia Univesidade Católica do Rio de Janeiro\\
      Rio de Janeiro, RJ 22453-900, Brazil \\
    \texttt{ milidiu@inf.puc-rio.br}\\
    \And 
    Rafael Henrique Santos Rocha \\
    Departmento de Informática\\
      Pontifícia Univesidade Católica do Rio de Janeiro\\
      Rio de Janeiro, RJ 22453-900, Brazil \\
    \texttt{ rhsantos@inf.puc-rio.br}\\
}

\begin{document}

\maketitle

\begin{abstract}
 The goal of Ordinal Regression is
	to find a rule that ranks items from a given set.
Several learning algorithms to solve this prediction problem
build an ensemble of binary classifiers.
Ranking by Projecting uses interdependent binary perceptrons.
These perceptrons share the same direction vector,
but use different bias values.
Similar approaches use independent direction vectors and biases.
To combine the binary predictions,
most of them adopt a simple counting heuristics.
Here,
we introduce a novel cumulative sum scoring function to combine
the binary predictions.
The proposed score value aggregates
	the strength of each one of the relevant binary classifications on how large is the item's rank. 
We show that our modeling casts ordinal regression as a Structured Perceptron problem.
As a consequence,
we simplify its formulation and description,
which results in two simple online learning algorithms.
The second algorithm is a Passive-Aggressive version of the first algorithm.
We show that under some rank separability condition
both algorithms converge.
Furthermore,
we provide mistake bounds for each one of the two online algorithms.
For the Passive-Aggressive version,
we assume the knowledge of a separation margin,
what significantly improves the corresponding mistake bound.
Additionally,
we show that Ranking by Projecting is a special case of our prediction algorithm.
From a neural network architecture point of view,
our empirical findings suggest a layer of cusum units for ordinal regression,
instead of the usual softmax layer of multiclass problems.

\end{abstract}
\section{Introduction}
\label{sec:introduction}

In the Ordinal Regression problem,
the general goal is to find a rule that ranks items from a given set
\citep{hang2011}.
For that sake,
it is assumed that each item from the set is described by a pair $(\xbold,y)$,
	given by $d$ numeric attributes $\xbold$ in $\mathbb{R}^d$
	and an integer $y$ in the rank set $Y = \{1,\cdots,r\}$ representing the item's rank.
Without loss of generality,
we assume that the last attribute has a constant value equal to $-1$,
that is,
$x_d = -1$ for all items.
This assumption helps to simplify notation when constructing linear discriminative models,
what would be the case in what follows.
The higher the rank value,
the higher the item's rank.
Hence,
we are assuming that the rank set has a total order.
For this problem,
we are required to provide a predictor $h$,
such that the value $\hat{y}$,
defined by
$
	\hat{y} = h(\xbold),
$
is close to the value $y$ for the pair $(\xbold,y)$.
The prediction accuracy is measured by a given loss function
$\e(y,\hat{y})$ with values in $\mathbb{R}$.
An illustrative example of loss function is
$\e(y,\hat{y}) = |y - \hat{y} |$. 

Several learning algorithms to solve the Ordinal Regression prediction problem
build an ensemble of $r$ binary classifiers.
The $k-$th binary task is
	to classify if a given item rank $y$ is equal to or larger than $k$.
Observe that the first binary task is trivial,
	since every rank $y$ is equal to or larger than $1$.
Ranking by Projecting \citep{crammer2002}
uses an ensemble of interdependent binary perceptrons to generate its prediction.
These perceptrons share the same direction vector,
but use different bias values.
Let the vector $\ubold$ represent the projection direction, 
	and the vector $\bbold$ represent the $r$ ordered thresholds
								$b_1 \leq b_2 \leq \cdots \leq b_r$,
	that are used to split the ranking classes along the direction $\ubold$,
with the first threshold fixed as $b_{1}= - \infty$.
Hence,
the binary perceptron parameters can be written as $\w_y = (\ubold,b_y)$.
These binary classifiers are simultaneously trained by the \emph{PRank} algorithm.
We restate the Ranking by Projecting predictor as
\begin{equation}
	\hat{y} = \max_{y \in Y, \w_y.\xbold \geq 0}{y}\;.
\end{equation}

In recent approaches to Ordinal Regression
\citep{li2007,niu2016},
the ensemble of binary perceptrons uses independent direction vectors and biases.
To combine these binary predictions,
they adopt a simple counting heuristics that we restate as
\begin{equation}
\label{eq:countingPredictor}
	\hat{y} = \sum_{k=1}^{r}{\indicator{\w_k .\xbold \geq 0}}
\end{equation}
where $\w_k$ is the direction vector of the $k-$th binary perceptron.
By observing that all ranks are equal to or larger than $1$,
we set $\w_1 = 0$ to be used as a sentinel value in the computations.

A consistency requirement for the counting predictor given by (\ref{eq:countingPredictor}) is that
$
	\indicator{\w_1 .\xbold \geq 0} \geq \indicator{\w_2 .\xbold \geq 0} \geq
	\cdots 	\geq \indicator{\w_r .\xbold \geq 0} \;\;.
$
For the ensemble with independent direction vectors,
this constraint is hard to satisfy.
Therefore,
it is usually relaxed by most learning schemes.
That is not the case for the \emph{Prank} algorithm,
	where this constraint is always satisfied,
	since it keeps $b_1 \leq b_2 \leq \cdots \leq b_r$
	throughout the whole learning process.

Similarly,
\cite{pedregosa2017} introduce a generalized counting predictor
that we restate as follows
\begin{equation}
\label{eq:consistentPredictor}
	\hat{y} = \sum_{k=1}^{r}{\indicator{\alpha_{k-1} < 0}}
\end{equation}
where $\alpha_k$ is the $k-$th ordering score,
and we set $\alpha_0 = -\infty$ to be used as a sentinel value in the computations.
To complete their predictor specification,
they add the constraint
$
    \alpha_0 \leq \alpha_1 \leq \cdots \leq \alpha_{r-1} \;\;.
$
They use Fisher consistent surrogate loss functions for learning this predictor.

\citet{antoniuk2013} propose
	a linear classification framework for ordinal regression,
which uses $m$ binary classifiers to compose the ensemble.
To solve the underlying learning problem,
they cast it as an unconstrained convex risk minimization problem,
for which many efficient solvers exist.
We call this approach as the \emph{General Linear Multiclass Ordinal Regression}
 (GLMORD).
We restate their predictor as
\begin{equation}
\label{eq:generallinearPredictor}
	\hat{y} = \arg\max_{k =1,\cdots, r}
										{ \sum_{j=1}^{m}{a_{kj}\w_j . \xbold} }\;.
\end{equation}

Here,
we introduce a novel cumulative sum scoring function $s$ to combine
the binary predictions.
It is defined as
\[
	s(\xbold,k,\w_1,\cdots,\w_r) = \sum_{j=1}^{k}{\w_j . \xbold}\;. 
\]
where $k=1,\cdots,r$.
The $s$ score value aggregates
	the strength of each one of the relevant binary classifications on how large is the item's rank.
This measure improves consistency among the binary predictions,
by taking into account how strong they are.
We call \emph{CuSum Rank} the resulting new predictor,
which chooses the rank with the highest score,
that is,
\begin{equation}
\label{eq:cusumPredictor}
	\hat{y} = \arg\max_{k =1,\cdots, r} \sum_{j=1}^{k}{\w_j . \xbold} \;. 
\end{equation}

Observe that
	(\ref{eq:cusumPredictor}) is a special case of (\ref{eq:generallinearPredictor}),
	with $m=r$ and $a_{kj} = \indicator{j \leq k}$.

The Structured Perceptron introduced by \cite{collins2002}
is an online learning framework that generalizes the binary perceptron \citep{rosenblatt1957}.
It has been used to build structured predictors for
	multiclass,
	free trees,
	arborescences,
	clusters,
	disjoint intervals,
	shortest paths,
	among many other structured outputs.
%\citep{fernandes2012}.
In this work,
we cast Ordinal Regression as a Structured Perceptron problem.
	
\iffalse	
We cast
	\emph{Prank},
	\emph{CuSum Rank},
	and Passive-Aggressive \emph{CuSum Rank}
	%and \emph{GLMORD}
as Structured Perceptrons.
As a consequence,
we obtain online learning algorithms for each one of them.
We show that under a rank separability condition,
we can apply the \emph{Structured Perceptron Theorem},
proving that
    the three corresponding algorithms converge.
Furthermore,
we provide mistake bounds for the online versions of these algorithms.
Our mistake-bound analysis is close to \cite{crammer2002},
but we generalize their result to the \emph{CuSum Rank} algorithm.
For the Passive-Aggressive version,
we assume the knowledge of a separation margin,
what significantly improves the corresponding mistake bound.
\fi

In Section \ref{sec:separability},
	we introduce the notion of rank linear separability.
In Section \ref{sec:cusum},
	we introduce \emph{CuSum Rank}
		-- the cumulative sum ranking algorithm.
In Section \ref{sec:structured},
	we introduce the loss augmented linearly separable structured problems.
	Then,
	we show a general mistake bound for the structured perceptron on these problems.
In Section \ref{sec:cusumrank},
	we cast the cumulative sum ranking algorithm as a Structured Perceptron,
		obtaining as a consequence its corresponding mistake bound.
In Section \ref{sec:ranking},
	we cast the \emph{Prank} algorithm also as a Structured Perceptron.
In Section \ref{sec:experiments},
	we perform some experiments to illustrate the \emph{CuSum Rank} approach.
Finally,
in Section \ref{sec:conclusion},
    we summarize and discuss our findings.
\section{Rank Linear Separability}
\label{sec:separability}
Some special conditions are sufficient
to ensure that an ensemble of perceptrons
consistently solves an Ordinal Regression problem.
Let us assume that we are given a dataset $\D$ for this problem.
A strong condition is that
	all the perceptrons in the ensemble correctly classify the examples in $\D$,
	with some positive margin $\delta > 0$.
\iffalse
The $k$-th binary perceptron task is
	to classify if a given item rank $y$ is equal to or larger than $k$.
Let $\w_k$ be the parameter vector of the $k$-th binary perceptron
and also $\w=(\w_1,\cdots,\w_r)$.
\fi
Hence,
for each example $(\xbold, y) \in \D$,
%we must have
%\[
		%\w_k . \xbold \geq   \delta   \;\;\;\;\;\;\;\;\;\;\;\;\;\textrm{for}\;\;    k \leq y
%\]
%and
%\[
			%\w_k . \xbold \leq  -\delta    \;\;\;\;\;\;\;\;\;\;\textrm{for}\;\;     		k   >  y		\;.
%\]
%Therefore,
%the $k$-th binary perceptron margin satisfies the following condition
\begin{equation}
	\label{eq:separability}
	\sign{(y - k)}. \w_k . \xbold \geq   \delta \;\;\;\;\;\;\;\;\;\;\;\;
											\textrm{for}\;\; 2 \leq k \leq r\;,
\end{equation}
where we omit the first perceptron since it is solving a trivial task.

When (\ref{eq:separability}) is satisfied
	for all $(\xbold, y) \in \D$
	and $\| \w \| = 1$,
we say that
	$\D$ is \emph{rank linearly separable}
	by $\w_1,\cdots,\w_r$
	with margin $\delta$.
Suppose we have
	$\ubold$ and $\bbold$ such that
	$b_1 \leq b_2 \leq \cdots \leq b_r$,
	$\| (\ubold, \bbold) \| = 1$
	and (\ref{eq:separability}) holds
	with $\w_k = (\ubold, b_k)$ for $k=1,\cdots,r$.
In this case,
we say that
	$\D$ is \emph{Prank linearly separable}
	by $\ubold$ and $\bbold$
	with margin $\delta$,
since we have
\[
		b_y + \delta \leq \sum_{j=1}^{d-1}{u_j.x_j} \leq b_{y+1} - \delta
\]
for each example $(\xbold, y) \in \D$.
Let $\D_0 = \{  ((0,0,-1),1),((0,1,-1),2),((1,1,-1),2),((1,0,-1),3), \}$ be a dataset.
It is easy to see that $\D_0$ is rank linearly separable,
but not Prank linearly separable.

Rank linear separability is a key condition for an ensemble
of perceptrons to solve ordinal regression.
%This is what we show in the following sections,
%by using the Structured Perceptron framework.
\section{Cumulative Sum Ranking}
\label{sec:cusum}
Let us assume that we are given a dataset $\D$ of examples $(\xbold, y)$ for an Ordinal Regression problem,
	where $\xbold \in \mathbb{R}^d$ and $y \in \{1,\cdots,r\}$.
Our approach here is to build an ensemble of $r$ binary perceptrons.
For $k=1,\cdots,r$,
	the $k$-th binary task is
		to classify if a given item rank $y$ is equal to or greater than $k$.
	
To define the $s$ score of a given rank value $k$ for an input $\xbold$,
we add the margin contributions from each one of the binary perceptrons,
what gives
\[
	s(\xbold,k;\w) = \sum_{j=1}^{r}{\sign{(k-j)}.\w_j . \xbold} \;,
\]
that is,
\begin{equation}
\label{eq:sumscore}
	s(\xbold,k;\w)
									%= \sum_{j=1}  ^{k}{\w_j . \xbold}
									%- \sum_{j=k+1}^{r}{\w_j . \xbold}
									= 2.\sum_{j=1}  ^{k}{\w_j . \xbold}
									- \sum_{j=1}^{r}{\w_j . \xbold} \;.
\end{equation}

From (\ref{eq:sumscore}),
we get that
%\[
	%\max_{1 \leq k \leq r} { s(\xbold,k;\w)	} =	
				%- \sum_{j=1}^{r}{\w_j . \xbold}
				%+ 2. \max_{1 \leq k \leq r} { \left( \sum_{j=1}  ^{k}{\w_j . \xbold} \right) } \;,
%\]
%that is,
\[
	\arg\max_{1 \leq k \leq r} {s(\xbold,k;\w)	} =
									\arg\max_{1 \leq k \leq r} { \left( \sum_{j=1}  ^{k}{\w_j . \xbold} \right) } \;.
\]

Hence,
our maximum score predictor is given by
\begin{equation}
\label{eq:cusumscore}
	\hat{y}  =  \arg\max_{1 \leq k \leq r} { \left( \sum_{j=1}  ^{k}{\w_j . \xbold} \right) } \;.
\end{equation}

The proposed ensemble of binary classifiers results in a \emph{cumulative sum} guided predictor.
From a neural network architecture point of view,
equation (\ref{eq:cusumscore}) suggests a layer of cusum units for ordinal regression,
whereas in the multiclass problem it is usual to have a softmax layer.
In Figure \ref{fig:cusumrank},
we outline the \emph{CuSum Rank} algorithm
	for online learning of the predictor parameters in (\ref{eq:cusumscore}).
Observe that $\w_1 = \0$ throughout the whole learning process,
	since it is never updated by the algorithm.
\begin{figure}[ht]
\small
\center
\fbox{\begin{minipage}{0.6\textwidth}
\begin{description} \itemsep0pt \parskip0pt \parsep0pt \itemindent-0.9cm
	\item \textbf{for} $k=1$ \textbf{to} $r$ 
 				\begin{description} \itemindent-1.2cm
    		 			\item $\w_k \leftarrow 0$
				\end{description}
  \item \textbf{for each new} $(\xbold,y) \in \D$
    		\begin{description} \itemindent-1.2cm
    					\item $\hat{y}  =  \arg\max_{1 \leq k \leq r} { \left( \sum_{j=1}  ^{k}{\w_j . \xbold} \right) }$
    					\item \textbf{if} $y \neq \hat{y}$
    					\begin{description} \itemindent-1.2cm
										\item \textbf{for} $k=\min(y,\hat{y})+1$ \textbf{to} $\max(y,\hat{y})$ 
			    		  				\begin{description} \itemindent-1.2cm
			    		  						\item $\w_k \leftarrow \w_k + \sign(y-\hat{y}).\xbold$
												\end{description}
							\end{description}						
    		\end{description}
    		\item \textbf{return}$(\w)$
\end{description}
\end{minipage}}
\caption{The \emph{CuSum Rank} algorithm.}
\label{fig:cusumrank}
\end{figure}

The Passive-Aggressive approach uses a different learning rule.
In Figure \ref{fig:passive_agressive_cusumrank},
we outline the loss sensitive Passive-Aggressive \emph{CuSum Rank} algorithm
	for online learning of the predictor parameters in (\ref{eq:cusumscore}),
	assuming that we know a separation margin $\delta$.

\begin{figure}[ht]
\small
\center
\fbox{\begin{minipage}{0.6\textwidth}
\begin{description} \itemsep0pt \parskip0pt \parsep0pt \itemindent-0.9cm
	\item \textbf{for} $k=1$ \textbf{to} $r$ 
 				\begin{description} \itemindent-1.2cm
    		 			\item $\w_k \leftarrow 0$
				\end{description}
  \item \textbf{for each new} $(\xbold,y) \in \D$
    		\begin{description} \itemindent-1.2cm
    					\item $\hat{y}  =  \arg\max_{1 \leq k \leq r} { \left( \sum_{j=1}  ^{k}{\w_j . \xbold} \right) }$
    					\item \textbf{if} $y \neq \hat{y}$
    					\begin{description} \itemindent-1.2cm
										\item $\bar{\w} \leftarrow 0$
										\item \textbf{for} $j=\min(y,\hat{y})+1$ \textbf{to} $\max(y,\hat{y})$ 
			    		  				\begin{description} \itemindent-1.2cm
			    		  						\item $\bar{\w} \leftarrow \bar{\w} + \w_j$
												\end{description}
												\item $\rho \leftarrow \frac{\sign(y-\hat{y}).\delta - \bar{\w}.\xbold}{|\hat{y} - y|.\|\xbold\|^2}  $
										\item \textbf{for} $j=\min(y,\hat{y})+1$ \textbf{to} $\max(y,\hat{y})$ 
			    		  				\begin{description} \itemindent-1.2cm
			    		  						\item $\w_j \leftarrow \w_j + \rho.\xbold$
												\end{description}
							\end{description}						
    		\end{description}
    		\item \textbf{return}$(\w)$
\end{description}
\end{minipage}}
\caption{The online loss sensitive Passive-Aggressive \emph{CuSum Rank} algorithm.}
\label{fig:passive_agressive_cusumrank}
\end{figure}

Both versions of the \emph{CuSum Rank} algorithm have been designed
	as online Structured Perceptrons.
%In the next sections,
%	we show this design rationale and its learning consequences.

\section{Structured Perceptron}
\label{sec:structured}

In machine learning based structured prediction,
	we learn a predictor $h$ from a training set $\D$
	of correct input-output pairs $(\xbold,\y)$.
It is expected that the predicted structure $\hat{\y}$,
given by
$
	\hat{\y} = h(\xbold),
$
provides a good approximation to the structure $\y$
of the corresponding example $(\xbold,\y) \in \D$.
The predictor quality is expressed by a loss function $\e(\y,\hat{\y})$:
	the smaller the loss,
	the better the prediction.

After observing a new example $(\xbold,\y) \in \D$,
the \emph{online} structured perceptron algorithm updates the weight vector $\w$ of a parameterized predictor given by
$$
	\hat{\y} = h(\xbold;\w) = \arg\max_{\y \in \Y(\xbold)} \w . \Ftr(\xbold,\y),
$$
	where
		$\Y(\xbold)$ is the set of feasible output structures for $\xbold$
		and $\Ftr(\xbold,\y)$ is a feature map of $(\xbold,\y)$.
In Figure \ref{fig:sperc},
we outline this algorithm.
The prediction $\hat{\y}$ is the solution of an optimization problem,
 	the so called {\em prediction problem}.
The objective function of this problem is given by $s$ 
	and scores candidate output structures for the given input.
\begin{figure}[ht]
\small
\center
\fbox{\begin{minipage}{0.6\textwidth}
\begin{description} \itemsep0pt \parskip0pt \parsep0pt \itemindent-0.9cm
  \item $\w \leftarrow \0$
    \item \textbf{for each new} $(\xbold,\y) \in \D$
    \begin{description} \itemindent-1.2cm
    		\item $\hat{\y} \leftarrow \arg\max_{\y \in \Y(\xbold)} \w . \Ftr(\xbold,\y) $
%    		\item $\hpred \leftarrow \arg\max_{h \in \Y(\xbold)} \ubold . \Ftr(\xbold,\h) +  C \cdot \e_r(\h,\hconstr)$
    		\item $\w\leftarrow \w + \Ftr(\xbold,\y) - \Ftr(\xbold,\hat{\y})$
    	\end{description}
\end{description}
\end{minipage}}
\caption{The \emph{online} structured perceptron algorithm.}
\label{fig:sperc}
\end{figure}

Let $\e$ be a loss function defined on $Y$.
We say that $\D$ is $\e$-augmented linearly separable by $\bar{\w}$
		with size $R$,
if and only if
there exist
	a vector $\bar{\w}$ on the feature space,
	with $\left\| \bar{\w }\right\| = 1$,
	and $R > 0$
such that
for each $(\xbold,\y) \in \D$ and $\y' \in Y-\{y\}$
we have
\[
	\begin{array}{ll}
		\textrm{(a)}&	\left\| \Ftr(\xbold,\y) - \Ftr(\xbold,\y') \right\|^2 \leq \e(\y,\y').R^2 \\
		\textrm{(b)}&	\bar{\w} . (\Ftr(\xbold,\y) - \Ftr(\xbold,\y')) \geq \e(\y,\y') \;.
	\end{array}
\]

Condition (b) states that $\D$ is linearly separable 
in the margin re-scaled formulation \citep{tsochantaridis2005,mcAllester2010}.

The next theorem,
    that we state without a proof,
    provides a mistake bound for this learning setup.

\begin{thm}
\label{th:theorem1}
Let $\D=\{(\xbold_i,\y_i)\}_{i=1}^{n}$ be a dataset,
	where each $(\xbold,\y) \in X \times Y$.
If
	$\D$ is $\e$-augmented linearly separable
		with radius $R$,
and $\hat{\y}_i$ is the \emph{online} structured perceptron prediction for $\xbold_i$,
then,
for $t=1,\cdots,n$,
we have that
\[
	\sum_{i=1}^{t}{\e(\y_i,\hat{\y}_i)} \leq R^2.
\]
\end{thm}

The generalization of Novikoff's Theorem to the Structured Perceptron \citep{collins2002}
follows from Theorem \ref{th:theorem1},
as stated next.

\begin{crl}
Let $\D=\{(\xbold_i,\y_i)\}_{i=1}^{n}$,
	with each $(\xbold,\y) \in X \times Y$ 
	and $\left\| \Ftr(\xbold,\y) - \Ftr(\xbold,\y') \right\| \leq R$
		for all $(\xbold,\y) \in \D$ and $\y' \in Y-\{y\}$.
Let also
	$\hat{\y}_i$ be the \emph{online} loss sensitive structured perceptron prediction for $\xbold_i$.
If $\D$ is linearly separable with margin $\delta$,
then for the online Structured Perceptron algorithm
\[
	\sum_{i=1}^{t}{\onebold[\y_i \neq \hat{\y}_i]} \leq \frac{R^2}{\delta^2}
\]
where
	$\onebold[\y \neq \hat{\y}]$ is the  \textsl{0-1} loss function.
\end{crl}

\iffalse
\begin{proof}
To get the bound,
we set the loss function $\e$ in Theorem \ref{th:theorem1} as $\e(\y,\hat{\y}) = \delta . \onebold[\y \neq \hat{\y}]$.
The result is then immediate.
\end{proof}
\fi

The \emph{Passive-Aggressive} variation of the
online structured perceptron algorithm
\citep{crammer06a}
uses a modified update rule,
given by
\[
	\w\leftarrow \w + \tau{(\xbold,\y)} . \left[ \Ftr(\xbold,\y) - \Ftr(\xbold,\hat{\y}) \right] \;,
\]
where the step size $\tau{(\xbold,\y)}$ is chosen such that
	the correct value $\y$ would be the solution of the $\arg \max$ prediction problem,
	for the corresponding updated value $\w$.
Here,
we adopt a loss sensitive step size,
that is,
\[
	\tau{(\xbold,\y)} =\frac{ \w . \Ftr(\xbold,\hat{\y} )-\w . \Ftr(\xbold,\y)  +  \e(\y,\hat{\y}) }
											    { \|\Ftr(\xbold,\y) - \Ftr(\xbold,\hat{\y})\|^2 } \;.
\]

\iffalse
In Figure \ref{fig:agressive},
we outline the \emph{online} loss sensitive Passive-Aggressive structured perceptron algorithm.
\begin{figure}[ht]
\small
\center
\fbox{\begin{minipage}{0.6\textwidth}
\begin{description} \itemsep0pt \parskip0pt \parsep0pt \itemindent-0.9cm
  \item $\w \leftarrow \0$
    \item \textbf{for each new} $(\xbold,\y) \in \D$
    \begin{description} \itemindent-1.2cm
    		\item $\hat{\y} \leftarrow \arg\max_{\y \in \Y(\xbold)} \w . \Ftr(\xbold,\y) $
    		\item $\tau \leftarrow \frac{\w . \Ftr(\xbold,\hat{\y} )-\w . \Ftr(\xbold,\y)  +  \e(\y,\hat{\y})}
																		{\|\Ftr(\xbold,\y) - \Ftr(\xbold,\hat{\y})\|^2} $
    		\item $\w\leftarrow \w + \tau . \left[ \Ftr(\xbold,\y) - \Ftr(\xbold,\hat{\y}) \right]$
    	\end{description}
\end{description}
\end{minipage}}
\caption{The \emph{online} loss sensitive Passive-Aggressive structured perceptron algorithm.}
\label{fig:agressive}
\end{figure}
\fi

The following theorem,
that we state without a proof,
provides the mistake bound for the loss sensitive Passive-Aggressive variation.

\begin{thm}
\label{th:theorem2}
Let $\D=\{(\xbold_i,\y_i)\}_{i=1}^{n}$ be a $\e$-augmented linearly separable dataset,
		with radius $R$,
Let also
	$\hat{\y}_i$ be the \emph{online} loss sensitive Passive-Aggressive structured perceptron prediction for $\xbold_i$.
Then,
for $t=1,\cdots,n$,
we have that
\[
	\sum_{i=1}^{t}{\e^2(\y_i,\hat{\y}_i)} \leq R^2.
\]
\end{thm}

The Passive-Aggressive structured perceptron mistake bound,
for a linearly separable dataset by a known margin,
follows from Theorem \ref{th:theorem2},
as stated next.

\begin{crl}
\label{crl:passive_aggressive_bound}
Let $\D=\{(\xbold_i,\y_i)\}_{i=1}^{n}$,
	with each $(\xbold,\y) \in X \times Y$ 
	and $\left\| \Ftr(\xbold,\y) - \Ftr(\xbold,\y') \right\| \leq R$
		for all $(\xbold,\y) \in \D$ and $\y' \in Y-\{y\}$.
Let also
	$\hat{\y}_i$ be the \emph{online} loss sensitive Passive-Aggressive structured perceptron prediction for $\xbold_i$.
If
	$\D$ is linearly separable with margin $\delta$
	and this value is known,
then
	for the online Passive-Aggressive structured perceptron algorithm
	\begin{equation}
		\label{eq:paMistakeBound}
		\sum_{i=1}^{t}{\onebold[\y_i \neq \hat{\y}_i]} \leq \frac{R^2}{\delta^4}
	\end{equation}
where
	$\onebold[\y \neq \hat{\y}]$ is the  \textsl{0-1} loss function.
\end{crl}

\iffalse
\begin{proof}
To get the bound,
we set the loss function $\e$ in the loss sensitive Passive-Aggressive algorithm  as
\[
	\e(\y,\y') = \delta . \onebold[\y \neq \y'] \;.
\]
The result is then immediate from Theorem \ref{th:theorem2},
since $\e^{2}(\y,\y') = \delta^{2} . \onebold[\y \neq \y']$
and
\[
	\left\| \Ftr(\xbold,\y) - \Ftr(\xbold,\y') \right\| \leq \onebold[\y \neq \y'] . R
	= \e(\y,\y') . \frac{R}{\delta} \;.
\]
\end{proof}
\fi

Although (\ref{eq:paMistakeBound}) is a very attractive mistake bound,
the corresponding algorithm modification is not of direct implementation,
since it requires prior knowledge of a separation margin $\delta$,
what is not usually found in practice.
The online structured perceptron algorithm
is a learning framework with four hot spots,
namely:
\iffalse
	\emph{(i)}   the $\w$ parameter vector,
	\emph{(ii)}  the set $\Y(\xbold)$ of feasible output structures,
	\emph{(iii)} the feature map $\Ftr(\xbold,\y)$,
	and \emph{(iv)} the algorithm that solves the $\arg\max$ combinatorial problem.
\fi
    $\w$,
    $\Y(\xbold)$,
    $\Ftr(\xbold,\y)$,
    and the $\arg\max$ solver.
Next,
	we show how to instantiate each one of these four hot-spots
		to obtain ordinal regression predictors and their properties.

\section{\emph{CuSum Rank} as a Structured Perceptron}
\label{sec:cusumrank}

Our goal is
	to build the predictor $\hat{y}$ defined by (\ref{eq:cusumscore}),
	by learning its underlying parameters $\w_1,\cdots,\w_r$
	    from a dataset $\D$.
For that sake,
	we instantiate the four hot spots of the Structured Perceptron framework,
	to show how we get the \emph{CuSum Rank} learning algorithm from it.
	
The first hot spot is
	the parameter vector $\w$,
that we define as
    $\w = (\w_1,\cdots,\w_r)$.
The second hot spot is
	the set $\Y(\xbold)$ of feasible output structures,
	which is simply $\{1,\cdots,r\}$.
The third hot spot is
	the feature map.
For $y =1,\cdots,r$,
we define the $r.d$-dimensional feature map $\Ftr(\xbold,y)$ as
    $\Ftr(\xbold,y) =
	            (\xbold,\cdots,\xbold,\0_{(r-y).d})$,
where
	$\0_{(r-y).d}$ is a $(r-y).d$-dimensional vector of zeros.
	
Since we have
	the parameter $\w$
	and the feature map $\Ftr(\xbold,y)$,
we can compute the score function as
\begin{equation}
\label{eq:cusumScore}
	\w . \Ftr(\xbold,y) = \sum_{j=1}^{y}{\w_j.\xbold} \;.
\end{equation}
The fourth hot spot is
	the $\arg \max$ problem,
	that provides the associated structured predictor as
\[
	\hat{y} = \arg\max_{k =1,\cdots, r}
										{ \sum_{j=1}^{k}{\w_j . \xbold} }\;,
\]
which is the same predictor defined by (\ref{eq:cusumscore}).
This is a trivial maximization problem since $r$ is fixed and small.

Observing that
$\Ftr(\xbold,y) - \Ftr(\xbold,\hat{y}) =
	\sign(y-\hat{y}) .
	(\0_{\min(y,\hat{y}).d},\xbold,\cdots,\xbold,\0_{(r-\max(y,\hat{y})).d})$,
we get the perceptron update rule as
\[
\w \leftarrow \w
    + \sign(y-\hat{y}) .
      (\0_{\min(y,\hat{y}).d},\xbold,\cdots,\xbold,\0_{(r-\max(y,\hat{y})).d}),
\]
that is,
\[
	\w_k \leftarrow \w_k + \sign(y-\hat{y}).\xbold,
									\ \ \ \ \textrm{for}\ \ \ \  k=\min(y,\hat{y})+1,\cdots,\max(y,\hat{y}).
\]
It is interesting to note that this rule never updates $w_1$,
which keeps its initial value.
This final comment completes our development of the \emph{CuSum Rank} algorithm
from the Structured Perceptron framework.
In Figure \ref{fig:cusumrank},
we outline this algorithm
	for online learning of the required predictor parameters.

\iffalse
It is also worth to note that
\begin{equation}
	\label{eq:cuSumRankUpperBound}
	\left\| \Ftr(\xbold,y) - \Ftr(\xbold,\hat{y}) \right\| ^2 =
								|y - \hat{y}| . \left\| \xbold \right\|^2
\end{equation}
and
\[
	\w. ( \Ftr(\xbold,y) - \Ftr(\xbold,\hat{y}) ) =								
				\sum_{j=\min(y,\hat{y})+1}^{\max(y,\hat{y})}{\sign(y-\hat{y}).\w_j.\xbold}
\]

Moreover,
if the dataset $\D$ is rank linearly separable by $\w_1,\cdots,\w_r$
with margin $\delta$,
we get that
%\[
	%\w. ( \Ftr(\xbold,y) - \Ftr(\xbold,\hat{y}) ) \geq								
				%\sum_{j=\min(y,\hat{y})+1}^{\max(y,\hat{y})}{\sign(y-\hat{y})^2 .  \delta}	
%\]
%that is,
\begin{equation}
	\label{eq:cuSumRankLowerBound}
	\w. ( \Ftr(\xbold,y) - \Ftr(\xbold,\hat{y}) ) \geq |y - \hat{y}| . \delta
\end{equation}
for all $(\xbold,y) \in \D$.
\fi

Next,
	we derive the \emph{CuSum Rank} mistake bound
	also as a consequence of Theorem \ref{th:theorem1}.

\begin{crl}
Let $\D=\{(\xbold_i,y_i)\}_{i=1}^{n}$ be a dataset of ranked items,
	where
		each $\xbold \in \mathbb{R}^{d}$,
		with $\left\| \xbold \right\| \leq R$,
		$x_d = -1$ for all $(\xbold,y) \in \D$
		and each $y$ is a rank value in the set $Y = \{ 1, \cdots,r \}$.
If $\D$ is rank linearly separable by $\w_1,\cdots, \w_r$ with margin $\delta$
then,
for $t=1,\cdots,n$,
we have that
\[
	\sum_{i=1}^{t}{|y_i - \hat{y}_i|} \leq \frac{R^2}{\delta^2},
\]
where $\hat{y}_i$ is the \emph{CuSum Rank} prediction for $\xbold_i$.
\end{crl}

\iffalse
\begin{proof}
Let
	$\e(y,y')=\delta . |y-y'|$ be the loss function defined on $Y$.
Since $\D$ is rank linearly separable by $\w_1,\cdots, \w_r$
with margin $\delta$,
it follows from equations (\ref{eq:cuSumRankUpperBound}) and (\ref{eq:cuSumRankLowerBound}) that
$\D$ is $\e$-augmented linearly separable,
with radius $R$.
From these facts,
the proposed mistake bound results as an immediate consequence of Theorem \ref{th:theorem1}.
\end{proof}
\fi

The above mistake bound is sharper than the \emph{Prank} mistake bound by
a factor of $(r-1)$.
This is not surprising,
	since \emph{CuSum Rank} has more free parameters than \emph{Prank}.
There is also a sharper mistake bound for the Passive-Aggressive \emph{CuSum Rank} algorithm,
that we state next.

\begin{crl}
Let $\D=\{(\xbold_i,y_i)\}_{i=1}^{n}$ be a dataset of ranked items,
	where
		each $\xbold \in \mathbb{R}^{d}$,
		with $\left\| \xbold \right\| \leq R$,
		$x_d = -1$ for all $(\xbold,y) \in \D$
		and each $y$ is a rank value in the set $Y = \{ 1, \cdots,r \}$.
If $\D$ is rank linearly separable by $\w_1,\cdots, \w_r$ with a known margin $\delta$
then,
for $t=1,\cdots,n$,
we have that
\[
	\sum_{i=1}^{t}{|y_i - \hat{y}_i|} \leq \frac{R^2}{\delta^4},
\]
where $\hat{y}_i$ is the online Passive-Aggressive \emph{CuSum Rank} prediction for $\xbold_i$.
\end{crl}

\iffalse
\begin{proof}
Immediate from Corollary \ref{crl:passive_aggressive_bound}.
\end{proof}
\fi
\section{\emph{Prank} as a Structured Perceptron}
\label{sec:ranking}
\iffalse
In the Ordinal Regression problem,
we are given a set $\D$ of items,
	each one described by a pair $(\xbold,y)$.
This pair is composed by
	$d$ numeric attributes $\xbold$ in $\mathbb{R}^d$
	and an integer $y$ in the rank set $Y = \{1,\cdots,r\}$ representing the item's rank.
Without loss of generality,
we assume that the last attribute as a constant value equal to $-1$,
that is,
$x_d = -1$ for all items.
This assumption helps to simplify notation when constructing linear discriminative models,
what would be the case here.
\fi

Ranking by Projecting uses an ensemble of interdependent binary perceptrons that
share the same direction vector $\ubold$,
but use different bias values
	$b_1 \leq b_2 \leq \cdots \leq b_r$,
	that are used to split the ranking classes along the direction $\ubold$,
    with the first threshold fixed as $b_{1}= -\infty$.
Hence,
the binary perceptron parameters can be written as $\w_y = (\ubold,b_y)$.
We restate their predictor as
\begin{equation}
\label{eq:prankPredictor}
	\hat{y} = \max_{y \in Y, \w_y.\xbold \geq 0}{y}\;.
\end{equation}
The predictor parameters can be learned online
by the \emph{PRank} algorithm.
Now,
let us cast \emph{PRank} as an online structured perceptron algorithm.
Since $\w_y = (\ubold,b_y)$,
	we also split $\xbold$ as $\xbold = (\zbold,  -1)$,
	where $\zbold$ is a $(d-1)$-dimensional vector.
Hence,
	the cumulative score function given by (\ref{eq:cusumScore}) simplifies to
\[
	\w . \Ftr(\xbold,y) = \sum_{j=1}^{y}{\w_j.\xbold} 
											= y.\ubold.\zbold - \sum_{k=1}^{y}{b_k}
											= \sum_{k=1}^{y}{(\ubold.\zbold -b_k)}.
\]

From this equation,
we derive
	the associated parameter vector $\w$
	and the underlying feature map $\Ftr$,
respectively as
        $\w = (\ubold,\bbold)$,
where
	$\ubold$ represents the \emph{PRank} projection direction, 
	$\bbold$ consists of the $r$ ordered thresholds $b_1 \leq b_2 \leq \cdots \leq b_r$,
		with the first threshold fixed as $b_{1} = - \infty$,
and as
\iffalse
Additionally,
the third hot spot must be set as follows.
For $y =1,\cdots,r$,
we define the feature map $\Ftr(\xbold,y)$ as
\fi
        $\Ftr(\xbold,y) = (y.\zbold,-\1{y},\0_{r-y})$,
where
	$\1{y}$ is a $y$-dimensional vector of ones
	$\0_{r-y}$ is a $(r-y)$-dimensional vector of zeros.
\iffalse
The second hot spot is the set $\Y(\xbold)$ of feasible output structures,
	which is simply $\{1,\cdots,r\}$.
To set the fourth hot spot,
	we observe that
\fi
Now,
observe that
$
	\w . \Ftr(\xbold,1) = \ubold.\zbold - b_1 > 0,
$
since $b_1 = - \infty$.
For $y =2,\cdots,r$,
we have
$
	\w . \Ftr(\xbold,y) - \w . \Ftr(\xbold,y-1)= \ubold.\zbold - b_y
$
and so
\[
	\w . \Ftr(\xbold,y) \geq \w . \Ftr(\xbold,y-1)
													\ \textrm{if and only if}\  \ubold.\zbold \geq b_y.
\]
Since
	$-\infty = b_1 \leq b_2 \leq \cdots \leq b_r$,
we obtain that 
\begin{equation}
\label{eq:argmax}
	\arg\max_{y\in \Y, \ubold.\zbold \geq b_y} {\w . \Ftr(\zbold,y)}
																= 
							\max_{y\in \Y, \ubold.\xbold \geq b_y} {y}.
\end{equation}

\iffalse
that is,
the $\arg\max$ is just a search problem in the ordered list $b_1\leq\cdots\leq b_r$.
This is trivially solved in $O(\lg(r))$ time
	through a binary search.
\fi

Observing that
$\Ftr(\xbold,y) - \Ftr(\xbold,\hat{y}) =
	((y-\hat{y}).\zbold,\0_{\min(y,\hat{y})},-\sign(y-\hat{y}).\1{|y-\hat{y}|},\0_{r-\max(y,\hat{y})})$,
we get the perceptron update rule as
\iffalse
\[
	\w \leftarrow \w
	+ ((y-\hat{y}).\zbold,\0_{\min(y,\hat{y})},-\sign(y-\hat{y}).\1{|y-\hat{y}|},\0_{r-\max(y,\hat{y})}),
\]
that is,
\fi
\begin{equation}
\label{eq:w_update}
	\ubold \leftarrow \ubold + (y-\hat{y}).\zbold
\end{equation}
and
\begin{equation}
\label{eq:a_update}
	b_k \leftarrow b_k -\sign(y-\hat{y}),
									\ \ \ \ \textrm{for}\ \ \ \  k=\min(y,\hat{y})+1,\cdots,\max(y,\hat{y}).
\end{equation}
It is interesting to note that this rule never updates $b_1$,
which keeps its initial value.

By applying equations (\ref{eq:argmax}),(\ref{eq:w_update}) and (\ref{eq:a_update})
to the online structured perceptron learning framework,
we obtain the \emph{PRank} algorithm.
%shown in Figure \ref{fig:prank}.
Its corresponding mistake bound,
    introduced by \cite{crammer2005},
    is an immediate consequence of Theorem 1.

\section{Structured Kernel}
% Why Linear separability is important
The proof of convergence for the \emph{CuSum Rank} depends on 
    the linear separability of the input features.
One method to explore non-linear transformation of the input features implicitly is to use the Kernel method
%TODO: Brief explanation of the Kernel method
The Kernel method was adapted for the Strcutured Perceptron by Collins and Duffy\citep{collins2002_2}, they define the dual structured perceptron as we can see at Figure \ref{fig:structkernel}. 

\begin{figure}[ht]
\small
\center
\fbox{\begin{minipage}{0.8\textwidth}
\begin{description} \itemsep0pt \parskip0pt \parsep0pt \itemindent-0.9cm

	\item \textbf{for} $i=1$ \textbf{to} $n$ 
 			\begin{description} \itemindent-1.2cm
		 			\item $\hat{y} \leftarrow 
		 			    \arg\max_{z \in \mathcal{Y}(x)}{
		 			         \sum_{(i,j)}\alpha_{i,j}(
		 			                 \Phi(x,z)\cdot\Phi(x_i,i) - \Phi(x,z)\cdot\Phi(x_i,j)
		 			             )
		 			      }$
		 			      
		 			 \item \textbf{if} $y \neq \hat{y}$
		 			 \begin{description} \itemindent-1.2cm
		 			        \item $\alpha_{y,\hat{y}} \leftarrow \alpha_{y,\hat{y}} + 1$
		 			 \end{description}
			\end{description}
			
\end{description}
\end{minipage}}
\caption{The Dual Structured Perceptron algorithm.}
\label{fig:structkernel}
\end{figure}

A drawback of this approach is the potentially high memory consumption to keep the parameter $\alpha_{i,j}$ for each predicted structure $j$.
Each input $x$ defines a set of possible atomic elements $\mathcal{A}$,
    the output structures $y$ are sets of atomic elements, so $y \in 2^{\mathcal{A}}$.
Therefore, the number of possible output structures is typically exponential on the number of atomic elements
To deal with this problem 
    we derive the feature map of a structure as the sum of the feature maps of its atomic elements
\[
    \Phi(x,y) = \sum_{a \in y} \phi(x,a)
\],
Our version of the Dual Structured Perceptron count the atomic elements instead of the structures, as we can see in Figure
\ref{fig:atom_structkernel}.

%TODO: Add description of (i,j) domain 

\begin{figure}[ht]
\small
\center
\fbox{\begin{minipage}{0.95\textwidth}
\begin{description} \itemsep0pt \parskip0pt \parsep0pt \itemindent-0.9cm

	\item \textbf{for} $i=1$ \textbf{to} $n$ 
 			\begin{description} \itemindent-1.2cm
		 			\item $\hat{y} \leftarrow
		 			    \arg\max_{z \in \mathcal{Y}(x)}{
		 			         \sum_{(i,j)}\left(
		 			            \sum_{a \in z\cap i}
		 			                \alpha_{a}\phi(x,a)\cdot\phi(x_i,a) - 
		 			            \sum_{a \in z\cap j}
		 			                \alpha_{a}\phi(x,a)\cdot\phi(x_i,a)\right)
		 			      }$
		 			      
		 			 \item \textbf{if} $y \neq \hat{y}$
		 			 \begin{description} \itemindent-1.2cm
		 			        \item \textbf{for} $a \in y$
		 			            \begin{description} \itemindent-1.2cm
		 			                \item $\alpha_a \leftarrow \alpha_a + 1$
		 			             \end{description}
		 			        \item \textbf{for} $a \in \hat{y}$
		 			            \begin{description} \itemindent-1.2cm
		 			                \item $\alpha_a \leftarrow \alpha_a - 1$
		 			             \end{description}
		 			 \end{description}
			\end{description}
			
\end{description}
\end{minipage}}
\caption{The Dual Structured Perceptron algorithm.}
\label{fig:atom_structkernel}
\end{figure}

\section{Experiments}
\label{sec:experiments}
To illustrate the practical performance of the \emph{CuSum Rank} approach,
we conduct some experiments
    % using its batch version,
    applied to benchmark datasets for ordinal regression
\cite{pedregosa2017}\cite{fathony2017},
    described by \cite{chu2005}
    and available online\footnote{\url{http://www.gatsby.ucl.ac.uk/~chuwei/ordinalregression.html}}.
    
We assume the Mean Absolute Error (MAE) as the prediction quality metric.
For each dataset,
we perform two learning steps:
    feature learning
    and ordinal regression.
First,
we train a single layer neural network,
    with $1.000$ neurons in the hidden layer,
to perform ordinary least squares regression.
For each example,
we use the output of the hidden layer generated by the trained regression network
as the example's new representation.
Finally,
we use this new representation as input to our \emph{Cusum Rank} perceptron,
    we implement
        the average version \cite{collins2002}
        with margin \cite{tsochantaridis2005}, to improve performance.

\iffalse
Next,
we present the experimental setup
and report our findings and remarks.
\fi
    
\subsection{Experimental Setup}

For each dataset,
    the target values were discretized into ordinal quantities using equal-length binning, there are two versions with five and ten quantiles respectively, as described by \cite{chu2005},
    the input features are standardized by using the training set global mean and variance for each feature or minmax normalization.
% Next,
% for the three smaller datasets,
    % we run LOOCV, 
We use the same 20 random partitions as described by \cite{chu2005}, and use 1 partition for parameter selection.
% One tenth of each training set fold is used for validation.
For the ordinary least squares regression,
    we use gradient descent until convergence on the validation set,
    with the learning rate arbitrarily set to 
    $0.001$.
We say that convergence is achieved
    when there is no improvement on the validation set MAE
    during 100 consecutive epochs.
For the ordinal regression,
    we choose one partition for parameter selection,
    the selected parameters are
        feature normalization method, 
        number of epochs for convergence,
        regularization parameters.
% For the smaller datasets,
% the total number of gradient descent epochs is
%     500 for regression,
%     200 for softmax
%         with learning rate arbitrarily set to 0.1
%     and 50 for each perceptron 
%         with margin arbitrarily set to 0.1.

\subsection{Results}

Our empirical findings are summarized in Table \ref{tab:ordinal_results_5} and \ref{tab:ordinal_results_10}.
For each candidate model,
    we report its cross validation MAE and standard error estimates.
The main finding is that
    \emph{Cusum Rank} is competitive to the alternative model.
    
We observe on Table \ref{tab:ordinal_results_5} and on Table \ref{tab:ordinal_results_10} the results of the model for the 5 bin partition and 10 bin partition datasets respectively. The benchmark is the SVM model with gaussian kernel from \cite{chu2005}.

% It is also worth to notice that,
%     the observed MAE for \emph{Cusum Rank},
%         evaluated in the ordinal regression task,
%     is significantly smaller than the one for feature extraction neural network,
%         evaluated in the easier ordinary regression task.
        
% From a neural network architecture point of view,
% our empirical findings suggest the use of
%     an output layer of cusum units for ordinal regression
    % instead of the usual softmax layer of multiclass problems.

\begin{table}[h]
    \centering
    
    \begin{tabular}{cc||cccc}
         \noalign{\hrule height 1.03pt}
         Dataset & Size & Benchmark & CuSum \\
         \noalign{\hrule height 0.5pt}
         abalone & 4,177& $0.229$  & \boldmath{$0.228$} \\
         diabetes & 43 & $0.746 $ & \boldmath{$0.680$} \\
         auto-mpg & 392 & $0.259$ & \boldmath{$0.251$} \\
         pyrimidines & 74&\boldmath{$0.45$}  & $0.522$ \\
         machine & 209 & $0.1915$ &  \boldmath{$0.1872$} \\
         wisconsin & 194& $1.003$  & \boldmath{$0.979$} \\
         stocks & 950 & \boldmath{$0.1081$}& $0.160$ \\
         triazines & 186 & \boldmath{$0.6977$} & $0.720$ \\
         housing & 506& \boldmath{$0.2672$} & $0.2766$ \\
         \noalign{\hrule height 1.03pt}

    \end{tabular}
    \caption{MAE values on the benchmark ordinal datasets with 5 bins}
    \label{tab:ordinal_results_5}
\end{table}

\begin{table}[h]
    \centering
    
    \begin{tabular}{cc||cccc}
         \noalign{\hrule height 1.03pt}
         Dataset & Size & Benchmark & CuSum \\
         \noalign{\hrule height 0.5pt}
         abalone & 4,177& \boldmath{$0.5160$}  & $0.5169$ \\
         diabetes & 43 & $2.4577$ & \boldmath{$1.076$} \\
         auto-mpg & 392 & \boldmath{$0.5081$} & $0.5765$ \\
         pyrimidines & 74&\boldmath{$0.9187$}  & $0.9729$ \\
         machine & 209 & \boldmath{$0.4398$} &  $0.4906$ \\
         wisconsin & 194& $2.1250$  & \boldmath{$2.1171$} \\
         stocks & 950 & $0.1804$ & \boldmath{$0.1778$} \\
         triazines & 186 & $1.2308$ & \boldmath{$1.2058$} \\
         housing & 506& \boldmath{$0.4971$} & $0.6075$ \\
         \noalign{\hrule height 1.03pt}
    \end{tabular}
    \caption{MAE values on the benchmark ordinal datasets with 10 bins}

    \label{tab:ordinal_results_10}
\end{table}

\section{Conclusion}
\label{sec:conclusion}
Building an \emph{Ensemble of Perceptrons} is an effective approach to solve the Ordinal Regression prediction problem.
Here,
	we follow this approach
	and propose two versions of the \emph{CuSum Rank} online learning algorithm.
These two new algorithms are designed by
	instantiating the Structured Perceptron framework.
We also introduce a new mistake bound
	for the Structure Perceptron algorithm learning,
	when applied to the class of loss-augmented linearly separable structured problems.
Additionally,
we derive a novel mistake bound for the Passive-Aggressive version.
These new mistake bounds are key to obtain
	mistake bounds for the two \emph{CuSum Rank} online learning algorithms.
\iffalse
Additionally,
we cast the \emph{Prank} algorithm as
	a special case of our Cumulative Sum Ranking approach.
For several structured prediction problems,
the set $Y$ of possible structures is finite.
In these cases,
any loss function defined on $Y$ is bounded.
Therefore,
to apply our mistake bound to such a setting,
it is critical to provide a feature map that is loss sensitive.
We think that finding these feature maps is a relevant new open problem.
\fi
From a neural network architecture point of view,
our empirical findings suggest
	an output layer of cusum units for ordinal regression,
	whereas in the multiclass problem it is usual to have a softmax output layer.
It would be interesting to explore the effect of cusum hidden layers.

\section*{Acknowledgements}
This material is based upon work supported by the Air Force Office of
Scientific Research under award number FA9550-19-1-0020.

\end{document}